\documentclass[12pt,a4paper]{article}

\usepackage{url}
\usepackage{amsmath}
\usepackage{amssymb}
\usepackage{amsthm}
\usepackage{natbib}
\usepackage{hyperref}
\hypersetup{
	colorlinks=true,
  linkcolor=red,          
  citecolor=blue,         
  filecolor=magenta,      
  urlcolor=cyan           
}
\usepackage[nottoc,notlot,notlof]{tocbibind}
\usepackage[top=2cm, bottom=2cm, left=2.2cm, right=2.2cm]{geometry}
\usepackage{mathtools}

\newtheorem{theorem}{Theorem}[section]
\newtheorem{lemma}[theorem]{Lemma}

\newtheorem{corollary}[theorem]{Corollary}

\theoremstyle{definition}

\newtheorem{definition}[theorem]{Definition}
\newtheorem{example}[theorem]{Example}
\newtheorem{conjecture}[theorem]{Conjecture}

\newtheorem{remark}[theorem]{Remark}

\newcommand{\set}[1]{\left\{ #1 \right\}}

\newcommand{\R}{\mathbb{R}}

\newcommand{\N}{\mathbb{N}}

\newcommand{\PP}{\mathbb{P}}

\newcommand{\id}{\textnormal{id}}
\DeclareMathOperator{\diag}{diag}

\DeclareMathOperator*{\argmin}{arg\,min}

\newcommand{\PR}[2][]{\mathop{\mathbb{P}}_{#1}\left( #2 \right)}

\newcommand{\Merg}{\M_d^{\textnormal{erg}}}
\newcommand{\Mirr}{\M_d^{\textnormal{irr}}}

\newcommand{\eps}{\varepsilon}

\newcommand{\lazy}{\mathcal{L}}
\newcommand{\M}{\mathcal{M}}

\newcommand{\TV}{_{\textnormal{TV}}}

\newcommand{\tmix}{t_{\mathsf{mix}}}

\newcommand{\pistar}{\pi_\star}

\newcommand{\gammaps}{\gamma_\text{ps}}

\usepackage{xcolor}

\newcommand{\beq}{\begin{eqnarray*}}
\newcommand{\eeq}{\end{eqnarray*}}
\newcommand{\beqn}{\begin{eqnarray}}
\newcommand{\eeqn}{\end{eqnarray}}

\title{\vspace{-2cm}On the $\alpha$-lazy version of Markov chains in estimation and testing problems}

\author{Sela Fried\thanks{A postdoctoral fellow in the Department of Computer Science at the Ben-Gurion University of the Negev. Research supported by ISF grant No.~1456/18.}  \\  \href{mailto:selaf@post.bgu.ac.il}{selaf@post.bgu.ac.il}  \and Geoffrey Wolfer\thanks{A JSPS International Research Fellow (Department of Computer and Information Sciences, Tokyo University of Agriculture and Technology).}   \\  \href{mailto:geo-wolfer@m2.tuat.ac.jp}{geo-wolfer@m2.tuat.ac.jp}  }
\date{}

\begin{document}

\maketitle

\begin{abstract}
Given access to a single long trajectory generated by an unknown irreducible Markov chain $M$, we simulate an $\alpha$-lazy version of $M$ which is ergodic. This enables us to generalize recent results on estimation and identity testing that were stated for ergodic Markov chains in a way that allows fully empirical inference. In particular, our approach shows that the pseudo spectral gap introduced by \citet{paulin2015concentration} and defined for ergodic Markov chains may be given a meaning already in the case of irreducible but possibly periodic Markov chains.
\end{abstract}
 
\section{Introduction}\label{intro}

Estimating the parameters of a discrete distribution and distinguishing whether an unknown distribution is identical to a reference one or is $\eps$-far from it under some notion of distance, assuming access to iid samples, are two classical and well studied problems in statistics and computer science (e.g. \citet{han2015minimax}, \citet{kamath2015learning}, \citet{orlitsky2015competitive}, \citet{batu2001testing}, \citet{valiant2017automatic} and the references therein). In contrast, research on the analogue problems assuming the samples are obtained from a single long trajectory generated by an unknown Markov chain is rather young. This work is concerned with the latter and adds to the still rather short list of works on the subject: \citet{wolfer2019estimating, wolfer2020minimax, wolfer2021statistical}, \citet{cherapanamjeri2019testing}, \citet{daskalakis2018testing}, \citet{chan2021learning}, \citet{fried2021identity} and \citet{hao2018learning}. We shall review the results of these works in the Related literature section.

Before we rigorously present our main results, let us somewhat informally sketch their motivation: In \citet[Theorem 3.1]{wolfer2021statistical} an estimator for the transition matrix $M$ of an unknown ergodic Markov chain was constructed, based on a single long trajectory generated by the Markov chain. In addition, they upper bounded the sample complexity of obtaining an estimation whose distance from $M$ under the infinity norm is less than a prescribed $\eps\in(0,1)$ with high probability. The assumption that the unknown Markov chain is ergodic is crucial in their sample complexity analysis since it relies on concentration bounds (cf. \citet{paulin2015concentration}) which assume ergodicity. We have noticed that it is possible to generalize their results to irreducible but possibly periodic Markov chains in the following way: Let $M$ be an irreducible Markov chain on $d$ states from which we can sample a single long trajectory of arbitrary length. Recall that if $\alpha\in(0,1)$ and $I$ denotes the identity matrix of size $d$, then $\lazy_\alpha(M)=\alpha I+(1-\alpha)M$ is called the $\alpha$-lazy version of $M$. It is well known that the $\alpha$-lazy version of an irreducible Markov chain is  ergodic. Crucially, it is possible to simulate $\lazy_\alpha(M)$ even when one only has access (as is the case here) to a black box that generates a single long trajectory from $M$ (cf. (6) of Section \ref{sec; 4}). This simulation results in a trajectory sampled from the ergodic Markov chain $\lazy_\alpha(M)$ on which we may apply the estimator mentioned above. Finally, we pull back the estimation of $\lazy_\alpha(M)$ to an estimation of $M$.
\bigskip

Several questions immediately arise from this procedure:

\begin{enumerate}
    \item [(a)] Can it be done with other estimation problems or even other statistical inference problems, e.g. identity testing? 
    \item [(b)] The $\alpha$-lazy version of a Markov chain is a special case of a matrix transformation. Can the procedure be done with other transformations?
    \item [(c)] An upper bound for the sample complexity of the original (ergodic) estimation problem involves (the inverse of) the pseudo spectral gap $\gammaps(M)$ of an ergodic Markov chain $M$. It follows that the sample complexity in the irreducible case is upper bounded by $\gammaps(\lazy_\alpha(M))$ for $M$ an irreducible Markov chain. Since $\gammaps(M)=0$ for an irreducible but periodic Markov chain $M$, no meaningful comparison between $\gammaps(M)$ and $\gammaps(\lazy_\alpha(M))$ is possible. Instead, we ask the following question: Suppose $M$ is ergodic. How are $\gammaps(M)$ and $\gammaps(\lazy_\alpha(M))$ related? An answer to this question quantifies the ``price" one has to pay for not being sure whether an unknown irreducible Markov chain is periodic or not.
\end{enumerate}

Our first two main results address the questions in (a) and (b) that are related to Markov chains estimation problems, by which we mean a tuple $(\Theta, \rho, \M^\circ)$ where $\Theta$ is a set, $\rho\colon\Theta\times\Theta\to\mathbb{R}_{+}$ is a distance function, i.e., for $x,y\in\Theta$ it holds: $\rho(x,y)=0$ if and only if $x=y$ (Notice that we neither assume that $\rho$ is symmetric nor that it satisfies the triangle inequality. Thus, in addition to metrics such as the total variation or the Hellinger distance, our machinery is applicable to, for example, the Kullback–Leibler divergence). The set $\M^\circ$ is a subset of the set of all transition matrices of irreducible Markov chains on $d$ states which we denote by $\Mirr$. The set $\Theta$ consists of the values that the parameter to be estimated can take and is accompanied by a map $\theta\colon \Mirr\to\Theta$. The distance function $\rho$ is used to measure the distance between the true value of the parameter and its estimation. Following the works mentioned above, in the setting we consider we have access to a single long trajectory generated from an unknown Markov chain $M$, started from an arbitrary state, and the purpose is to construct an estimator for some parameter of interest $\theta(M)$ such that the distance between the estimation and the true value of the parameter is less than a prescribed $\eps\in(0,1)$. Along with the estimator we would like to know in advance how long the trajectory must be for the estimation to succeed. The minimal necessary trajectory length is referred to as the sample complexity. Since the estimation relies on a random trajectory, there is generally some  probability that the trajectory obtained is not informative. Thus, in addition to $\eps$, the sample complexity depends also on $\delta\in(0,1)$ such that the estimator is guaranteed to succeed with probability at least $1-\delta$.

The following theorem provides a sufficient condition for the feasibility of the procedure we described above. It uses the notions of a solution to an  estimation problem and of extendibility of their solutions which will be defined in Section \ref{sec; est}:

\begin{theorem}\label{thm; i1}
Consider a Markov chains estimation problem $(\Theta, \rho, \M^\circ)$. Let $\varphi\colon \Mirr\to\Mirr, g\colon\Theta\to\Theta$ and $\ell\colon (0,1)\to (0,1)$ be such that for every $x\in\Theta, M\in\varphi^{-1}(\M^\circ)$ and $\eps\in(0,1)$ it holds 
\begin{equation}\label{eq; 6}
\rho(x,\theta(\varphi(M)))<\ell(\eps)\Longrightarrow \rho(g(x),\theta(M))<\eps.\end{equation} Then every solution to $(\Theta, \rho, \M^\circ)$ extends to $\varphi^{-1}(\M^\circ)$.
\end{theorem}

Let us demonstrate a typical usage of Theorem \ref{thm; i1} by applying it on the problem of estimating the transition matrix mentioned above: We have 
$$\Theta = \Mirr,\;\theta(M) = M,\;\;\forall M\in\Mirr,\; \rho(M, \bar{M}) = ||M-\bar{M}||_\infty
, \;\;\forall M,\bar{M}\in\Theta.$$ For $\M^\circ$ we take $\M_{\pistar,\gammaps}$ which is defined as the set of all ergodic Markov chains on $d$ states whose minimum stationary probability and pseudo spectral gap are lower bounded by $\pistar,\gammaps\in(0,1)$, respectively. The set $\M_{\pistar,\gammaps}$ is the default choice in this work since the sample complexity of the problems that we consider was shown to be upper bounded by expressions that involve $\frac{1}{\gammaps}$ and $\frac{1}{\pistar}$. Taking $\varphi=\lazy_\alpha$ and $\ell(\eps) = (1-\alpha)\eps$ for every $\eps\in(0,1)$, a natural candidate for $g$ is $\lazy_\alpha^{-1}$ but it may well result in a matrix that is not irreducible. Even worse, we might end up with negative entries on the main diagonal, i.e., not a row-stochastic matrix. In order to guarantee an estimation which is an irreducible Makrov chain, $\lazy_\alpha^{-1}$ is followed by a projection on the set of all row-stochastic matrices of size $d$ which we denote by $\M_d$, which, in turn, is followed by an additional adjustment to guarantee irreducibility. We discuss this example in detail in Example \ref{ex; 722}(b).

Our second main result shows that there are cases in which extending solutions is hard:

\begin{theorem}\label{thm; i2}
Let $(\Theta, \rho, \M^\circ)$ be a Markov chains estimation problem with a finite sample complexity and let $\varphi\colon \Mirr\to\Mirr$ be such that there exist
 $M_1,M_2\in\varphi^{-1}(\M^\circ)$ with $\theta(\varphi(M_1))=\theta(\varphi(M_2))$ but $\theta(M_1)\neq\theta(M_2)$. Let $g\colon\Theta\to\Theta$ be such that for every $R>0$ there exists $\sigma>0$ with the following property: Whenever $x,y\in\Theta$ satisfy $\rho(x,y)<\sigma$ then $\rho(g(x),g(y))<R$ (this holds, for example, if $g$ is uniformly continuous). Then no solution to $(\Theta, \rho, \M^\circ)$ extends to $\varphi^{-1}(\M^\circ)$ via $g$.
\end{theorem}

In Example \ref{ex; 11} we shall show that the assumptions of Theorem \ref{thm; i2} are satisfied in the setting of \citet[Theorem 8.1]{wolfer2019estimating} who constructed an estimator for the pseudo spectral gap (in absolute error).

Our third main result is a reformulation of Theorem \ref{thm; i1} to Markov chains identity testing problems. In this setting one needs to decide, based on a single long trajectory generated from an unknown Markov chain, whether a certain parameter of the Markov chain equals a specific value or is $\eps$-far from it under some notion of distance for $\eps\in(0,1)$. We will apply the theorem on the Markov chains identity testing problem considered by \citet[Theorem 4.1]{wolfer2020minimax} in which the parameter of interest is the transition matrix. The formal definition of Markov chains identity testing problems and of extendibility of their solutions is given in Section \ref{sec; 5}.

\begin{theorem}\label{thm; i3}
Let $(\Theta, \rho, \M^\circ, \overline{\M^\circ})$ be a Markov chains identity testing problem.
Let $\varphi\colon \Mirr\to\Mirr$ and $\ell \colon (0,1)\to (0,1)$ be such that for every $ M\in\varphi^{-1}(\M^\circ), \overline{M}\in\varphi^{-1}(\overline{\M^\circ})$ and $\eps\in(0,1)$ it holds 

\begin{equation}\label{eq; 50}
\begin{cases}
\theta(M)=\theta(\overline{M})&\Longrightarrow\theta(\varphi(M))=\theta(\varphi(\overline{M}))\\\rho(\theta(M),\theta(\overline{M}))
>\eps&\Longrightarrow \rho(\theta(\varphi(M)),\theta(\varphi(\overline{M})))>\ell(\eps).
\end{cases}
\end{equation}
Then every solution to $(\Theta, \rho, \M^\circ, \overline{\M^\circ})$ extends to $(\varphi^{-1}(\M^\circ), \varphi^{-1}(\overline{\M^\circ}))$.
\end{theorem}

Our last main result addresses the question in (c), namely, the connection between $\gammaps(M)$ and $\gammaps(\lazy_\alpha(M))$ for an ergodic Markov chain $M$. Recall that the pseudo spectral gap captures the mixing time (cf. (\ref{eq; 11})) and that there exists a universal constant $c$ such that
$$\tmix(\lazy_\alpha(M))\leq\frac{c}{1-\alpha}\max\left\{\frac{1}{1-\alpha},\tmix(M)\right\}$$ (Lemma \ref{lem; 440}).
We conjecture (Conjecture \ref{con; 10}) that a similar inequality holds for the pseudo spectral gap. As a consequence of the following theorem we show that for an ergodic Markov chain $M$ such that $\gammaps(M)$ is obtained  at $k=1$ (cf. (\ref{eq; 64})) it holds \begin{equation}\label{eq; 44}\gammaps(\lazy_\alpha(M))\geq(1-\alpha)\gammaps(M).\end{equation} This confirms the conjecture 
for this class of ergodic Markov chains that contains all ergodic and reversible Markov chains.

In the following theorem
$\gamma(M)$ denotes the spectral gap of a reversible Markov chain $M$ and $M^\dagger$ denotes the multiplicative reversibilization of an irreducible Markov chain $M$:

\begin{theorem}\label{thm; i4}
Let $M$ be an ergodic Markov chain. Then $$\gamma(\lazy_\alpha(M)^\dagger)\geq(1-\alpha)\gamma(M^\dagger).$$
\end{theorem}

\bigskip

The paper is structured as follows: After a short literature review, in Section \ref{sec; a} we recall definitions and well known facts regarding Markov chains and their $\alpha$-lazy versions that will be used throughout this work. Our main reference here is \citet{levin2017markov}; In section \ref{sec; b}, that contains our main results, we formulate the framework and demonstrate its usage; Section \ref{sec; d} is devoted to the sample complexity aspect of moving to an $\alpha$-lazy version.

\subsection{Related literature}

In recent years there have been attempts to go beyond the iid assumption in statistical inference problems. The Markovian setting in which every sample depends only on the sample immediately preceding it is a natural weakening of the iid assumption. Works on the subject differ mainly in their choice of the distance notion between Markov chains: \citet{wolfer2020minimax, wolfer2021statistical} consider each row of the transition matrix separately and under the total variation distance between distributions. This leads to a distance notion based on the infinity norm. Their analysis relies on the concentration bounds of \citet{paulin2015concentration} involving the pseudo spectral gap which is defined only for ergodic Markov chains. \citet{chan2021learning} 
generalized Wolfer and Kontorovich's results to all irreducible Markov chains by recognizing that the sample complexity under the infinity norm is characterized by what they refer to as the $k$-cover time denoted by $t_\text{cov}^{(k)}$ and defined to be the expected length of a trajectory such that every state is visited at least $k$ times. Taking a different path, \citet{daskalakis2018testing} and \citet{cherapanamjeri2019testing} use a distance notion that relies on the largest eigenvalue of the geometric mean of the transition matrices. They make the stringent assumption that the Markov chains are symmetric (\citet{fried2021identity} subsequently showed that, in essence, reversibility suffices).

Our work is especially related to \citet{chan2021learning} since it follows from both theirs and ours that the results of Wolfer and Kontorovich mentioned above apply to all irreducible Markov chains: \citet{chan2021learning} achieve this by considering $t_\text{cov}^{(k)}$ that is defined for every irreducible Markov chain and we by moving to the $\alpha$-lazy version.

We argue that while the $k$-cover provides a sharper sample complexity characterization of the problem of inference in the single long trajectory setting, compared to the lower and upper bounds of Wolfer and Kontorovich that are based on the mixing properties and exhibit a logarithmic gap, at the moment it is not clear how to exploit this sharpness. Indeed, first, there is currently no method to estimate the $k$-cover time from a single long trajectory directly. Second, all the upper bounds on the $k$-cover times in terms of well known quantities that \citet{chan2021learning} provide (e.g. $t_\text{cov}^{(k)}=\tilde{O}\left( t_\text{cov}+\frac{k}{\pistar}\right)$ for irreducible Markov chains and $ t_\text{cov}^{(k)}=\tilde{O}\left( \frac{1}{\pistar\gammaps}+\frac{k}{\pistar}\right)$ for ergodic Markov chains) involve the minimum stationary probability. To the best of our knowledge, Example \ref{ex; 722} (a) (see also Lemma \ref{lem; 648}), which is a combination of our machinery and the estimator of \citet{wolfer2019estimating}, provides the only fully empirical method for the estimation of the minimum stationary probability of all irreducible Markov chains from a single long trajectory. Furthermore, \citet{wolfer2019estimating} provide a similar estimator of the pseudo spectral gap. It seems that although considerable progress has been made regarding the computational problem of approximating the cover time (e.g. \citet{feige2003deterministic}, \citet{bui2007compute}, \citet{feige2009deterministic} and \citet{ding2011cover}), the statistical estimation of the cover time from a single trajectory remains a challenging open question.

\section{Preliminaries}\label{sec; a}

\subsection{Markov chains}

Let $\Omega$ be a finite state space which we assume without loss to be equal to $[d] = \{1,2,\ldots,d\}$ where $2\leq d\in\N$. We denote by $\M_d$ the set of all $d\times d$ row-stochastic matrices and refer to elements of $\M_d$ as \emph{Markov chains}. For a matrix $M$ of size $d\times d$ and $i,j\in[d]$ we write $M(i,j)$ for the entry at the $i$th row and the $j$th column of $M$. Recall that a Markov chain  $M$ is called \emph{irreducible} if for every $i,j\in[d]$ there exists $t\in\N$ such that $M^t(i,j)>0$. The \emph{period} of a state $i\in [d]$ is defined as $$\text{gcd}\{t\in\N\;|\;M^t(i,i)>0\}.$$ If a Markov chain $M$ is irreducible, all its states share the same period (\citet[Lemma 1.6]{levin2017markov}). In this case, the \emph{periodicity} of $M$ is defined as the period of one of its states. If the periodicity is $1$, then $M$ is said to be \emph{aperiodic} and otherwise \emph{periodic}. Irreducible and aperiodic Markov chains are called \emph{ergodic}.

We denote by $\Mirr$ (resp. $\Merg$) the subset of $\M_d$  consisting of all irreducible (resp. ergodic) Markov chains. Let  $M\in\M_d$ and  $\alpha\in(0,1)$ (to be used henceforth). Following \citet{montenegro2006mathematical} and \citet{Hermon2016MaximalIA}, we call $\lazy_\alpha(M) = \alpha I + (1-\alpha)M$ the \emph{$\alpha$-lazy version of $M$} where $I$ is the identity matrix of size $d$. The simplex of all probability distributions over $[d]$ will be denoted by $\Delta_d$. If  $\mu\in\R^d$ and $i\in[d]$ we write $\mu(i)$ for the $i$th entry of $\mu$. Vectors are assumed to be row-vectors.

Let $M\in\M_d,\mu\in\Delta_d, m\in\N$ and $i_1,\ldots,i_m\in[d]$. By $(X_1,\ldots,X_m)\sim(M,\mu)$ we mean  $$\PP((X_1,\ldots,X_m)=(i_1,\ldots,i_m))=\mu(i_1)\prod_{t=1}^{m-1}M(i_t,i_{t+1}).$$

We say that $\pi\in\Delta_d$ is a \emph{stationary distribution} of  $M\in\M_d$ if $\pi M=\pi$. All Markov chains on a finite state space have at least one stationary distribution (\citet[Section 1.7]{levin2017markov}) and irreducible markov chains have unique stationary distributions (\citet[Corollary 1.17]{levin2017markov}). Furthermore, if $\pi$ is the stationary distribution of $M\in\Mirr$, then the \emph{minimum stationary probability of} $M$ defined as $$\pistar(M) = \min_{i\in[d]}\{\pi(i)\}$$ satisfies $\pistar(M)>0$ (\citet[Proposition 1.14]{levin2017markov}). Ergodic Markov chains converge to their stationary distribution in total variation (\citet[Theorem 4.9]{levin2017markov}), i.e., if $M\in\Merg$ and $\pi$ is its stationary distribution, then $$d(t):=\max_{i\in[d]}||M^t(i,\cdot)-\pi||_\text{TV}\underset{t\to\infty}{\longrightarrow}0$$ where $$||\mu-\nu||_\text{TV}=\frac{1}{2}\sum_{i\in[d]}|\mu(i)-\nu(i)|, \;\;\forall \mu,\nu\in\Delta_d$$  (\citet[Proposition 4.2]{levin2017markov}). For $\eps>0$ denote $$\tmix(M, \eps)=\min\{t\in\N\;|\;d(t)\leq\eps\}$$ and define the \emph{mixing time of $M$} by $\tmix(M) :=\tmix(M, 1/4).$

Let $M\in\Mirr$ with stationary distribution $\pi$. We write $Q (M)=\diag(\pi)M$ for the \emph{edge measure} of $M$ where $\diag(\pi)$ is the diagonal matrix whose entries correspond to $\pi$ and call $M$ \emph{reversible} if $Q (M)$ is symmetric. Suppose now $M\in\Merg$ is reversible. Then $1$ is an eigenvalue of $M$ with a one dimensional eigenspace and all eigenvalues of $M$ are real and lie in $(-1,1]$ (\citet[Lemmas 12.1 and 12.2]{levin2017markov}). They may be therefore ordered: $$1=\lambda_1>\lambda_2\geq\cdots \geq\lambda_d>-1.$$ The \emph{spectral gap} and the \emph{absolute spectral gap of} $M$ are defined by $$\gamma(M)=1-\lambda_2,\;\;\;\text{and}\;\;\;\gamma_\star(M) = 1 - \max\{\lambda_2, |\lambda_d|\}, \text{  respectively}.$$
The absolute spectral gap and the minimum stationary probability of $M$ capture $\tmix(M)$ (\citet[Theorems 12.3 and 12.4]{levin2017markov}): $$\left(\frac{1}{\gamma_\star(M)}-1\right)\ln 2\leq\tmix(M)\leq\frac{1}{\gamma_\star(M)}\ln\frac{4}{\pistar(M)}.$$

For ergodic Markov chains that are not necessarily reversible there exists an analogue of the absolute spectral gap, namely the \emph{pseudo spectral gap} that was introduced by \citet{paulin2015concentration}: Let $M\in\Merg$ with stationary distribution $\pi$. Define the \emph{time reversal of} $M$ by $M^* = \diag(\pi)^{-1}M^T\diag(\pi)$ and the \emph{multiplicative reversibilization of} $M$ by $M^\dagger = M^*M$. The pseudo spectral gap of $M$ is then defined to be
 \begin{equation}\label{eq; 64}\gammaps(M)=\max_{k\in\N}\left\{\frac{\gamma\left((M^k)^\dagger\right)}{k}\right\}.\end{equation} By \citet[Proposition 3.4]{paulin2015concentration}, \begin{equation}\label{eq; 11}\frac{1}{\gammaps(M)}\leq \tmix(M)\leq \frac{1}{\gammaps(M)}\left(1+2\ln2+\ln\frac{1}{\pistar(M)}\right).\end{equation}
For $\pistar,\gammaps\in(0,1)$ denote by $\M_{\pistar, \gammaps}$ the set of all ergodic Markov chains whose minimum stationary probability and pseudo spectral gap are lower bounded by $\pistar$ and $\gammaps$, respectively.

If $\pi$ is the stationary distribution of some irreducible Markov chain and $\mu\in\Delta_d$, define $$||\mu/\pi||^2_{2,\pi}=\sum_{i\in[d]}\frac{\mu(i)^2}{\pi(i)}.$$

Finally, if $A$ is a square matrix of size $d$ then the \emph{infinity norm of} $A$ is defined by $||A ||_\infty=\max_{i\in[d]}\sum_{j\in[d]}|A(i,j)|.$

\subsection{The \texorpdfstring{$\alpha$}-lazy version of a Markov chain}\label{sec; 4}

While Theorems \ref{thm; i1}, \ref{thm; i2} and \ref{thm; i3} are stated for arbitrary maps $\varphi\colon\Mirr\to\Mirr$, we are mainly interested in the case $\varphi = \lazy_\alpha$ where $\lazy_\alpha\colon \Mirr\to\Mirr$ is defined by $$\lazy_\alpha(M)=\alpha I+(1-\alpha)M, \;\;\forall M\in\Mirr.$$
Let us recall some of the properties of $\lazy_\alpha$. To this end fix $M\in\Mirr$ with stationary distribution $\pi$:
\begin{enumerate}
    \item [(1)] $\lazy_\alpha(\Mirr)\subseteq\Merg$. \item [(2)] $\lazy_\alpha$ is injective and is not surjective.
    \item [(3)] The stationary distribution of $\lazy_\alpha(M)$ is also $\pi$. In particular, $\pi_\star(\lazy_\alpha(M))=\pi_\star(M)$.
    \item [(4)] There exist $\pistar,\gammaps\in(0,1)$ such that $M\in\lazy_\alpha^{-1}(\M_{\pistar,\gammaps})$.
    \item [(5)] If $M$ is reversible, then so is $\lazy_\alpha(M)$.

\item [(6)] It is possible to simulate $\lazy_\alpha(M)$ even when one only has access  to a black box that generates a single long trajectory from $M$. Indeed, independently in each step, we toss a biased coin that comes out head with probability $1-\alpha$. If the coin shows head, we move one step along the trajectory from $M$. Otherwise, we insert a duplicate of the current state at the position immediately after the current position and set it to be the new current position.
\end{enumerate}

\section{Main results}\label{sec; b}

Definition \ref{def; 1873}, together with its analogue - Definition \ref{def; 222} - provide the basis for this work. The reader might find it useful to keep the following story in mind while reading them: Suppose we are interested in a certain parameter of irreducible Markov chains, e.g. their minimum stationary probability. Let $\Theta$ denote the set of all values that this parameter may take, e.g. $(0,1/2]$. Denote by $\theta(M)$ the value that the parameter takes for $M\in\Mirr$. Let $\hat{\theta}$ be an estimator that upon receiving a trajectory from an unknown $M\in\Mirr$, outputs an estimation $\hat{\theta}(M)$ of $\theta(M)$ whose ``quality" is measured by a distance function $\rho$. Suppose we know that for some class $\M^\circ\subseteq\Mirr$ the estimator needs not more than $m_0\in\N$ samples in order to output a ``good" estimate with ``high" probability. Now, suppose we want to use this estimator for Markov chains not belonging to $\M^\circ$. If there is a map $\varphi$ into $\M^\circ$, the following recipe might work: Given $M\in\Mirr$, apply $\hat{\theta}$ on $\varphi(M)\in\M^\circ$ and then pull back the estimation $\hat{\theta}(\varphi(M))$ to an estimation of $\theta(M)$.



\subsection{Estimation problems}\label{sec; est}

In this section we define Markov chains estimation problems, their solutions and extendibility of solutions. We then prove Theorems \ref{thm; i1} and \ref{thm; i2} and apply these on the results of \citet{wolfer2019estimating, wolfer2021statistical}.

\begin{definition}\label{def; 1873}
Let $\Theta$ be a set and $\theta\colon\Mirr\to\Theta$. Let $\rho\colon\Theta\times\Theta\to\R_+$ be a distance function and  $\M^\circ\subseteq\Mirr$. We refer to $(\Theta, \rho,\M^\circ)$ as a \emph{Markov chains estimation problem}.
An \emph{estimator} is a sequence $\left(\hat{\theta}_m\right)_{m\in\N}$ such that $\hat{\theta}_m\colon[d]^m\to\Theta$ for every $ m\in\N$. Now, let $m\in\N, M\in \Mirr, \mu\in\Delta_d, \eps, \delta
\in(0,1)$ and $\hat{\theta}$ an estimator. We write $\hat{\theta}_m(M)$ for the output of the estimator $\hat{\theta}_m$ upon receiving a sequence $(X_1, \ldots,X_m)\sim(M,\mu)$. Define the \emph{minimax risk}   
$$\mathcal{R}_m(\hat{\theta}, \eps, \M^\circ, \mu) =  \sup_{M \in \M^\circ} \PR{\rho(\hat{\theta}_m(M), \theta(M)) \geq \eps} $$ 
and the \emph{sample complexity} 
$$
m_0(\hat{\theta}, \eps,\delta, \M^\circ, \mu) = \inf \set{m \in \N\mid \mathcal{R}_m(\hat{\theta},\eps, \M^\circ,\mu) < \delta }.$$
An estimator $\hat{\theta}$ is a \emph{solution to $(\Theta, \rho, \M^\circ)$} if $m_0(\hat{\theta}, \eps, \delta, \M^\circ,\mu)<\infty$ for every $\eps,\delta\in(0,1)$. Finally, let $\varphi\colon\Mirr\to\Mirr$. We say that a solution $\hat{\theta}$ to $(\Theta, \rho, \M^\circ,\mu)$ \emph{extends to $\varphi^{-1}(\M^\circ)$} if there exists $g\colon\Theta\to\Theta$ such that $g\circ\hat{\theta}\circ\varphi$ is a solution to
 $(\Theta, \rho, \varphi^{-1}(\M^\circ),\mu)$.
\end{definition}


\paragraph{Proof of Theorem \ref{thm; i1}}

Let $\hat{\theta}$ be a  solution to $(\Theta, \rho, \M^\circ)$. We need to show that $g\circ\hat{\theta}\circ\varphi$ is a solution to $(\Theta, \rho, \varphi^{-1}(\M^\circ))$. To this end, let $\eps,\delta\in(0,1)$. For every $m\in\N$ it holds 
\begin{align}
    \mathcal{R}_{m}(\hat{\theta},\ell(\eps),\M^{\circ}, \mu)=&\sup_{M\in\M^{\circ}}\PP(\rho(\hat{\theta}_m(M),\theta(M))\geq\ell(\eps))\nonumber\\ \geq&\sup_{M\in\varphi^{-1}(\M^{\circ})}\PP(\rho(\hat{\theta}_m(\varphi(M)),\theta(\varphi(M)))\geq\ell(\eps))\nonumber\\ \geq & \sup_{M\in\varphi^{-1}(\M^{\circ})}\PP(\rho(g(\hat{\theta}_m(\varphi(M))),\theta(M))\geq\eps)\nonumber \\ = & \mathcal{R}_{m}(g\circ\hat{\theta}\circ\varphi,\eps,\varphi^{-1}(\M^{\circ}),\mu)\nonumber.
\end{align} It follows that $$\mathcal{R}_{m}(\hat{\theta},\ell(\eps),\M^{\circ},\mu)<\delta \Longrightarrow \mathcal{R}_{m}(g\circ\hat{\theta}\circ\varphi,\eps,\varphi^{-1}(\M^{\circ}),\mu)<\delta.$$ Therefore, since
$m_0(\hat{\theta},\ell(\eps),\delta,\M^\circ,\mu)<\infty$, then also $$m_0(g\circ\hat{\theta}\circ\varphi,\eps,\delta,\varphi^{-1}(\M^\circ),\mu)<\infty.$$
\qed

\begin{example}\label{ex; 722}
\begin{enumerate}
    \item [(a)] In \citet[Theorem 5.1]{wolfer2019estimating} a procedure for estimating the minimum stationary probability of an ergodic Markov chain (in relative error) was constructed, thus, in our terminology, a solution to the following Markov chain estimation problem: Let $\pistar, \gammaps \in (0,1)$ and $ \M^\circ = \M_{\pistar, \gammaps}$. Define
$$ \Theta = (0,1/2],\; \theta(M) = \pistar(M), \;\;\forall M\in\Mirr, \;\rho(x, y) = \left|\frac{x}{y}-1\right|, \;\;\forall x,y\in\Theta.$$ We shall now show that any solution to $(\Theta, \rho, \M^\circ)$ extends to $\lazy_\alpha^{-1}(\M^\circ)$. To this end, take $g=\id_\Theta$ and $\ell=\id_{(0,1)}$. Then $g$ and $\ell$ satisfy condition (\ref{eq; 6}). Indeed, let $x\in\Theta, M\in\lazy_\alpha^{-1}(\M^\circ)$ and $\eps\in (0,1)$ such that $ \left|\frac{x}{\pistar(\lazy_\alpha(M))}-1\right|<\eps$. Since $\pistar(\lazy_\alpha(M)) = \pistar(M)$ we also have $\left|\frac{x}{\pistar(M)}-1\right|<\eps$.

\item [(b)] In \citet[Theorem 3.1]{wolfer2021statistical} a procedure for estimating the transition matrix of an ergodic Markov chain was constructed, thus, a solution to the following Markov chains estimation problem (actually, Wolfer and Kontorovich did not make sure that the output of their estimator is ergodic, or even irreducible, but only a row-stochastic matrix. Nevertheless, this is easily achieved as we shall see immediately): Let $\pistar, \gammaps \in (0,1)$ and $\M^\circ = \M_{\pistar, \gammaps}$. Define
\begin{align}\Theta = & \Mirr,\;\theta(M) = M,\;\; \forall M\in\Mirr, \nonumber\\ \rho(M_1, M_2) = & ||M_1-M_2||_\infty,\;\;\forall M_1,M_2\in\Theta.\nonumber\end{align}
We shall now show that any solution to $(\Theta, \rho, \M^\circ)$ extends to  $\lazy_\alpha^{-1}(\M^\circ)$.  Defining $\ell\colon(0,1)\to(0,1)$ is straightforward: $$\ell(\eps)=\frac{(1-\alpha)\eps}{2},\; \; \forall\eps\in(0,1).$$
To justify the definition of $g\colon\Theta\to\Theta$ that we subsequently give, let us notice that for $M\in\lazy_\alpha^{-1}(\M^\circ)$, a solution $\hat{\theta}$ to $(\Theta,\rho,\M^\circ)$ gives an estimation $\hat{M}$ of $\lazy_\alpha(M)$, and $\hat{M}$ is a row-stochastic matrix. But it might happen that $\hat{M}\not\in\lazy_\alpha(\Mirr)$. Thus, applying $\lazy_\alpha^{-1}$ to $\hat{M}$ might lead to negative entries on the main diagonal and therefore to a matrix which is not row-stochastic. To mitigate this, one possibility is to project $\lazy_\alpha^{-1}(\hat{M})$ onto $\M_d$ with respect to $||\cdot||_\infty$ which is equivalent to projecting each of $\lazy_\alpha^{-1}(\hat{M})$'s rows on $\Delta_d$ with respect to the $\ell_1$ norm. This leads to a convex optimization problem whose details we give in \ref{app; 1} in the appendix. It remains to make sure that the row-stochastic matrix is irreducible. To this end denote by $\mathbf{1}$ the square matrix of size $d$ that has all entries equal to $1$ and for $N\in\M_d$ and $\beta\in(0,1)$ let $N_{\mathbf{1},\beta}=\frac{\beta}{d}\mathbf{1}+(1-\beta) N$. Clearly, $N_{\mathbf{1},\beta}$ is an irreducible (even ergodic) Markov chain. Furthermore, \begin{equation}\label{eq; 99}||N_{\mathbf{1},\beta}-N||_\infty=\beta\left|\left|\frac{1}{d}\mathbf{1}-N\right|\right|_\infty\leq 2\beta.\end{equation} Thus, for $M\in\Theta$ we define $$g(M)=\left(P_{\M_d} \left(\lazy_\alpha^{-1}(M)\right)\right)_{\mathbf{1},\eps/4}$$ where by $P_{\M_d} \left(\lazy_\alpha^{-1}(M)\right)$ we mean that the procedure from Lemma \ref{lem; 567} is applied independently to each row of $\lazy_\alpha^{-1}(M)$.
Then $g$ and $\ell$ satisfy condition of (\ref{eq; 6}). Indeed, let $\hat{M}\in\Theta, M\in\lazy_\alpha^{-1}(\M^\circ)$ and $\eps\in (0,1)$ such that $ ||\hat{M}-\lazy_\alpha(M)||_\infty<\frac{(1-\alpha)\eps}{2}$. This inequality is equivalent to $||\lazy_\alpha^{-1}(\hat{M})-M||_\infty<\frac{\eps}{2}.$ By Lemma \ref{lem; 567}, inequality (\ref{eq; 99}) and the triangle inequality, $||g(\hat{M})-M||_\infty<\eps$.
\end{enumerate}
\end{example}

\begin{remark}
Our definition of a solution to a Markov chains estimation problem only assumes that the sample complexity is finite and therefore sheds no light on the extent to which the sample complexity might change while extending a solution. We postpone the treatment of this  aspect to Section \ref{sec; d}.
\end{remark}

We now come to the 

\paragraph{Proof of Theorem \ref{thm; i2}}

Let $\hat{\theta}$ be a solution to $(\Theta, \rho, \M^\circ)$. By assumption, $$R:=\rho(\theta(M_1),\theta(M_2))>0.$$ Then there exists $\sigma\in(0,1)$ such that if $x,y\in\Theta$ satisfy $\rho(x,y)<\sigma$, then $\rho(g(x), g(y))<R/2$. Assume $g\circ\hat{\theta}\circ\varphi$ is a solution to $(\Theta, \rho, \varphi^{-1}(\M^\circ))$. Thus, for   $m\in\N$ large enough, we have $$\mathcal{R}_m\left(\hat{\theta}, \frac{\sigma}{2}, \M^\circ,\mu\right)<\frac{1}{4}\;\;\; \textnormal{and}\;\;\; \mathcal{R}_m\left(g\circ\hat{\theta}\circ\varphi, \frac{R}{4}, \varphi^{-1}(\M^\circ),\mu\right)<\frac{1}{4}.$$ It follows that, with probability at least $1/2$,
$$ \rho\left(\hat{\theta}_m(\varphi(M_1)), \hat{\theta}_m(\varphi(M_2))\right)<\sigma.$$ By the assumption on $g$, with probability at least $1/2$, $$\rho\left(g(\hat{\theta}_m(\varphi(M_1))), g(\hat{\theta}_m(\varphi(M_2)))\right)<R/2.$$
On the other hand, with probability at least $1/2$, for $i=1,2$:
$$\rho(g(\hat{\theta}_m(\varphi(M_i))), \theta(M_i)))< R/4.$$ This means that, with probability at least $1/2$,
$$\rho\left(g(\hat{\theta}_m(\varphi(M_1))), g(\hat{\theta}_m(\varphi(M_2)))\right)>R/2.$$ A contradiction.
\qed

\begin{example}\label{ex; 11}
In \citet[Theorem 8.1]{wolfer2019estimating} a procedure for estimating the pseudo spectral gap (in absolute error) of an ergodic Markov chain was constructed, thus, a solution to the following Markov chain estimation problem: Let $\pistar, \gammaps \in (0,1)$ and $ \M^\circ = \M_{\pistar, \gammaps}$. Define
$$ \Theta = [0,1],\;\theta(M) = \gammaps(M),\;\;\forall M\in\Mirr, \;\rho(x, y) = |x-y|, \;\;\forall x,y\in\Theta.$$ We claim that no solution $\hat{\theta}$ to  $(\Theta, \rho, \M^\circ)$ extends to $\lazy_{1/2}^{-1}(\M^\circ)$ with $g$ continuous (together with the compactness of $\Theta$ this means that $g$ is uniformly continuous). Indeed, by Theorem \ref{thm; i2}, it suffices to find $M_1, M_2 \in \lazy_{1/2}^{-1}(\M^\circ)$ such that $\gammaps(\lazy_{1/2}(M_1))=\gammaps(\lazy_{1/2}(M_2))$ but $\gammaps(M_1)\neq\gammaps(M_2)$. Consider $$ M_1 = \begin{pmatrix}
0 & 1 & 0\\
0.5 & 0 & 0.5 \\
0 & 1 & 0
\end{pmatrix},\; M_2 = \begin{pmatrix}
0.5 & 0.2 & 0.3\\
0.5 & 0.2 & 0.3\\
0.5 & 0.2 & 0.3
\end{pmatrix}.$$
Then $\gammaps(\lazy_{1/2}(M_1))=0.75= \gammaps(\lazy_{1/2}(M_2))$ but $\gammaps(M_1)=0\neq 1= \gammaps(M_2)$. To see that, first notice that $M_1$ is periodic and therefore $\gammaps(M_1)=0$ (cf. \citet[p. 11]{paulin2015concentration}). Let $\pi_i$ denote the stationary distribution of $M_i, i=1,2$. It holds
$$\pi_1 = (0.25, 0.5, 0.25) \text{ and } \pi_2 = (0.5, 0.2, 0.3).$$ Thus, with $Q_i = Q(M_i), i=1,2$ we have
$$Q_1 = \begin{pmatrix}
0 & 0.25 & 0\\
0.25 & 0 & 0.25 \\
0 & 0.25 & 0
\end{pmatrix}, 
Q_2 = \begin{pmatrix}
0.25 & 0.1 & 0.15\\
0.1 & 0.04 & 0.06\\
0.15 & 0.06 & 0.09
\end{pmatrix}.$$
Conclude that both $M_1$ and $M_2$ are reversible and therefore also $\lazy_{1/2}(M_1)$ and $\lazy_{1/2}(M_2)$. By \citet[p. 18]{wolfer2021statistical}, for ergodic and reversible Markov chains, the maximum in the definition of the pseudo spectral gap is obtained at $k=1$. Thus, for $M\in\{\lazy_{1/2}(M_1), M_2, \lazy_{1/2}(M_2)\}$: $$\gammaps(M)=\gamma(M^2).$$ The claim follows now by calculating the second largest eigenvalue of $M^2$ in each of the three cases. \end{example}

\subsection{Identity testing problems}\label{sec; 5}

In this section we define Markov chains identity testing  problems, their solutions and extendibility of solutions. We then apply Theorem \ref{thm; i3} on the results of \citet[Theorem 4.1]{wolfer2020minimax}.

\begin{definition}\label{def; 222}
Let $\Theta$ be a  set and  $\theta\colon \Mirr\to\Theta$. Let $\Theta\times\Theta\to\mathbb{R}_+$ be a distance function and $\M^\circ,\overline{\M^\circ}\subseteq\Mirr$. We refer to $(\Theta,\rho,\M^\circ,\overline{\M^\circ})$ as a Markov chains identity testing problem. An identity tester is a sequence  $\left(\hat{\tau}_m\right)_{m\in\N}$ such that $\hat{\tau}_m\colon[d]^m\times\M_d \to\{0,1\}$ for every $m\in\N$. Now, let $m\in\N, M, \overline{M}\in \Mirr, \mu\in\Delta_d, \eps, \delta
\in(0,1)$ and $\hat{\tau}$ an identity tester. We will write $\hat{\tau}_m(M, \overline{M})$ for the output of the identity tester $\hat{\tau}_m$ upon receiving a sequence $(X_1,\ldots,X_m)\sim(M, \mu)$ and $\overline{M}$. Define the \emph{minimax risk}   
$$
\mathcal{R}_m(\hat{\tau},\eps, \M^\circ,\overline{\M^\circ}, \mu)= \sup_{\substack{M\in\M^\circ,\\\overline{M} \in \overline{\M^\circ}}} \PR{ \begin{cases}
\hat{\tau}_m(M,\overline{M})\neq0, & \text{if } \theta(M)=
\theta(\overline{M})\\
&\text{or}\\
\hat{\tau}_m(M,\overline{M})\neq1, & \text{if } \rho(\theta(M),\theta(\overline{M}))>\eps\end{cases}}$$
and the \emph{sample complexity} 
$$
m_0(\hat{\tau}, \delta, \eps, \M^\circ,\overline{\M^\circ},\mu) = \inf\set{m \in \N\mid \mathcal{R}_m(\hat{\tau},\eps, \M^\circ,\overline{\M^\circ},\mu) < \delta }.$$ An identity tester $\hat{\tau}$ is a \emph{solution to $(\Theta, \rho, \M^\circ, \overline{\mathcal{M}^\circ})$} if $$m_0(\hat{\tau}, \delta, \eps, \mathcal{M}^\circ, \overline{\mathcal{M}^\circ},\mu)<\infty$$ for every $\eps,\delta\in(0,1)$.

Let $\varphi\colon\M_d\to\M_d$ such that $\varphi^{-1}(\M^\circ),\varphi^{-1}(\overline{\M^\circ})\subseteq U$ . We say that a solution $\hat{\tau}$ to $(\Theta, \rho, \M^\circ, \overline{\mathcal{M}^\circ})$ \emph{extends to $(\varphi^{-1}(\M^\circ), \varphi^{-1}(\overline{\M^\circ}))$} if $\hat{\tau}\circ\varphi$ is a solution to
$(\Theta, \rho, \varphi^{-1}(\M^\circ), \varphi^{-1}(\overline{\mathcal{M}^\circ})).$
\end{definition}

The proof of Theorem \ref{thm; i3} is similar to the proof of Theorem \ref{thm; i1} and is omitted.

\begin{example}\label{ex; 198} In \citet[Theorem 4.1]{wolfer2020minimax} a procedure for identity testing of Markov chains transition matrices was constructed, thus, a solution to the following Markov chains identity testing problem: Let $\pistar, \gammaps \in (0,1)$ and $\M^\circ,\overline{\M^\circ}=\M_{\pistar,\gammaps}$. Define
\begin{align}
\Theta = \Mirr, \theta(M)=&M, \;\;\forall M\in\Mirr, \nonumber\\ \rho(M_1, M_2) =& ||M_1-M_2||_\infty, \;\;\forall M_1,M_2\in\Theta.\nonumber\end{align}
Then any solution to $(\Theta, \rho, \M^\circ, \overline{\M^\circ})$ extends to $(\varphi^{-1}(\M^\circ), \varphi^{-1}(\overline{\M^\circ}))$. Indeed, $\ell\colon(0,1)\to(0,1)$ given by $$\ell(\eps)=(1-\alpha)\eps,\; \; \forall\eps\in(0,1)$$ 
obviously satisfies condition (\ref{eq; 50}) of Theorem \ref{thm; i3}.
\end{example}

\section{The cost of moving to an \texorpdfstring{$\alpha$}{1}-lazy version of a Markov chain}\label{sec; d}

In the previous section we were solely interested in the feasibility of extending solutions to Markov chains estimation and testing problems. In this section we wish to highlight that extending a solution might result in an increase in the sample complexity of the problem. In particular, we give the exact statements of the problems whose solutions we extended.

\begin{remark}\label{rem; 121}
Let $M\in\M_d$. It follows from the way the Markov chain $\lazy_\alpha(M)$ is simulated 
(cf. (6) of Section \ref{sec; 4}) that in order to obtain the actual number of samples coming from $M$, the length of the simulated Markov chain should be multiplied by approximately $1-\alpha$. In order to make this observation more quantitative, let us denote by $m_\text{act}$ the number of samples coming from $M$. Thus, if we simulate $\lazy_\alpha(M)$ for $m\in\N$ steps, then $m_\text{act}\sim\text{Binomial}(m, 1-\alpha)$. In particular, $\mathbb{E}[m_\text{act}]=(1-\alpha)m$ and for every $\rho>0$, by Hoeffding's inequality,  $$\PP\left(m_\text{act}\geq (1-\alpha+\rho)m\right)\leq e^{-2\rho^2 m}.$$
\end{remark}

\subsection{The sample complexity of some extended solutions}\label{sec; 22}

The bounds given in this section are for the simulated $\alpha$-lazy Markov chain. The  observations made in Remark \ref{rem; 121} may be used to calculate the actual sample size from the original Markov chain.

\begin{lemma}[Example \ref{ex; 722}(a)]\label{lem; 648}
 Let $M\in\Mirr, \eps, \delta\in(0,1)$ and $\mu\in\Delta_d$. Then there exist a universal constant $c$ and an estimator $\hat{\theta}$ for $\pi_\star(M)$ such that if $$m\geq\frac{c}{(1-\alpha)^2\eps^2\pi_\star(M)}\frac{\ln\frac{1}{\pi_\star(M)\delta}}{\gammaps(\lazy_\alpha(M))}$$ then, upon receiving a sequence $(X_1, \ldots,X_m)\sim(M,\mu)$, it holds $$\left|\frac{\hat{\theta}(M)}{\pi_\star(M)}-1\right|<\eps,$$ with probability at least $1-\delta$.
\end{lemma}

\begin{lemma}[Example \ref{ex; 722}(b)]\label{lem; 649}
Let $M\in\Mirr,\eps, \delta\in(0,1)$ and $\mu\in\Delta_d$. Let $\pi$ be the stationary distribution of $M$. Then there exist a universal constant $c$ and an estimator $\hat{\theta}$ for $M$ such that if $$m\geq c\max\left\{\frac{1}{(1-\alpha)^2\eps^2\pi_\star(M)}\max\left\{d,\ln\frac{1}{(1-\alpha)\eps\delta}\right\}, 
\frac{\ln\frac{d||\mu/\pi||_{2,\pi}}{\delta}}{\gammaps(\lazy_\alpha(M))\pi_\star(M)}\right\}$$ then, upon receiving a sequence $(X_1, \ldots,X_m)\sim(M,\mu)$, it holds $$||\hat{\theta}(M)-M||_\infty<\eps,$$ with probability at least $1-\delta$.
\end{lemma}

\begin{lemma}[Example \ref{ex; 198}]
Let $\eps, \delta \in (0,1),M,\overline{M}\in\Mirr$ and $\mu\in\Delta_d$. There exist a universal constant $c$ and an identity tester $\hat{\tau}$
such that if 
$$m\geq \frac{c}{\pi_\star(\overline{M})}\max\left\{\frac{\sqrt{d}}{(1-\alpha)^2\eps^2}\ln\frac{d}{\delta(1-\alpha)\eps}, 
\tmix(\lazy_\alpha(\overline{M}))\ln\frac{d}{\delta\pi_\star(\overline{M})}\right\}$$ 
then, upon receiving a sequence $(X_1, \ldots,X_m)\sim(M,\mu)$ and $\overline{M}$, it holds
\begin{align}
M = \overline{M} & \Longrightarrow \hat{\tau}(M, \overline{M})=0\nonumber \\
||M - \overline{M}||_\infty>\eps & \Longrightarrow \hat{\tau}(M, \overline{M})=1, \nonumber
\end{align} with probability at least $1-\delta$.
\end{lemma}

\subsection{Bounds on the mixing time and the pseudo spectral gap of the \texorpdfstring{$\alpha$}{1}-lazy version}

A natural question that arises is by how much the sample complexity can increase if we unnecessarily move to the $\alpha$-lazy version, i.e., the Markov chain is already ergodic. The following Lemma, that bounds the mixing time of the $\alpha$-lazy version in terms of $\alpha$ and the mixing time of the original Markov chain, seems to be folklore. We give its proof for completeness in (\ref{proof; 1}) in the Appendix.

\begin{lemma}\label{lem; 440}
Let $M\in\M_d^\textnormal{erg}$ and let $\eps\in(0,1)$. Then $$\tmix(\lazy_\alpha(M), \eps)\leq\max\left\{\frac{2\ln\frac{2}{\eps}}{(1-\alpha)^2},\frac{2}{(1-\alpha)}\tmix(M,\eps/2)\right\}.$$ In particular, $$\tmix(\lazy_\alpha(M))\leq\max\left\{\frac{5}{(1-\alpha)^2},\frac{6}{1-\alpha}\tmix(M)\right\}.$$
\end{lemma}

A similar result to Lemma \ref{lem; 440} seems to hold for the pseudo spectral gap:

\begin{conjecture}\label{con; 10} 
There exists a function $f\colon(0,1)^2\to\mathbb{R}_+$ such that for every $M\in\mathcal{M}_d^\textnormal{erg}$ it holds $\gammaps(\lazy_\alpha(M))\geq f(\alpha,\gammaps(M))$.
\end{conjecture}

As a consequence of the following lemma we establish Conjecture \ref{con; 10} for a certain class that contains all ergodic and reversible Markov chains.

\paragraph{Proof of Theorem \ref{thm; i4}}

 Let $\pi$ be the stationary distribution of $M$. Denote $$\mathcal{F}=\{f\in\R^d\mid\sum_{i=1}^d f(i)\pi(i)=1 \text{ and } \sum_{i=1}^d f(i)^2=1\}.$$ By \citet[Lemmas 13.11 and 13.12]{levin2017markov} and noticing that $\gamma(I)=0$, we have
\begin{align}
\gamma(\lazy_\alpha(M)^\dagger)
&=\min_{f\in\mathcal{F}} \frac{1}{2}\sum_{i,j\in[d]}(f(i)-f(j))^2 \pi(i)\lazy_\alpha(M)^\dagger(i,j) \nonumber \\&=\min_{f\in\mathcal{F}} \frac{1}{2}\sum_{i,j\in[d]}(f(i)-f(j))^2 \pi(i)\cdot\nonumber\\&\hspace{2.5cm}\left(\alpha^2I  + \alpha(1-\alpha)(M+M^*) + (1-\alpha)^2M^\dagger\right) (i,j)\nonumber \\  &\geq 2\alpha(1-\alpha)\gamma\left(\frac{1}{2}(M+M^*)\right)+(1-\alpha)^2\gamma(M^\dagger). \label{eq; 17}
\end{align}
It is well known (e.g. \citet[Theorem 9.F.4]{marshall1979inequalities}) that $$\left(\lambda_2\left(\frac{1}{2}(M+M^*)\right)\right)^2 \leq \lambda_2(M^\dagger).$$ Thus, $$
\gamma(M^\dagger)\leq  \gamma\left(\frac{1}{2}(M+M^*)\right)\left(1+\lambda_2\left(\frac{1}{2}(M+M^*)\right)\right)\leq 2\gamma\left(\frac{1}{2}(M+M^*)\right).$$
It follows that 
$$
(\ref{eq; 17})\geq \alpha(1-\alpha)\gamma(M^\dagger)+(1-\alpha)^2\gamma(M^\dagger) = (1-\alpha) \gamma(M^\dagger).\nonumber
$$
\qed

\begin{corollary}\label{cor; 223}
Let $M\in \Merg$. Suppose that $\gammaps(M)$ is obtained at $k=1$ (cf. (\ref{eq; 64})). Then
$$\gammaps(\lazy_\alpha(M))\geq(1-\alpha)\gammaps(M).$$ This holds, in particular, if $M$ is reversible. 
\end{corollary}

\begin{proof}
It holds $$\gammaps(\lazy_\alpha(M))\geq \gamma(\lazy_\alpha(M)^\dagger)\geq (1-\alpha)\gamma(M^\dagger) =(1-\alpha)\gammaps(M).$$
That $\gammaps(M)$ is obtained at $k=1$ if $M$ is reversible was shown in \citet[p. 18]{wolfer2021statistical}. \end{proof}

\begin{remark}
The class $$\{M\in\Merg\mid\gammaps(M) \text{ is obtained at } k=1\},$$ of which ergodic and reversible Markov chains are a subclass, seems to be very large. Indeed, for $M\in \Merg$ it suffices that
$\gamma(M^\dagger)\geq\frac{1}{2}$. This follows from the more general observation that if 
 $\gamma( (M^{k_0}) ^\dagger) \geq 1 - \frac{1}{k_0+1}$ for some $k_0\in\N$ then $\gammaps(M)$ is obtained at $k\leq k_0$.
\end{remark}

\begin{remark}
In the case of ergodic and reversible Markov chains a stronger result holds: By \citet[Theorem 3.3]{paulin2015concentration}, one may replace (at the price of changing the universal constant) $\gammaps(\lazy_\alpha(M))$ with $\gamma(\lazy_\alpha(M))$ in Lemmas \ref{lem; 648} and \ref{lem; 649}. Now, 
\begin{align} 
\gamma(\lazy_\alpha(M)) =& 1 - \lambda_2(\lazy_\alpha(M))\nonumber\\
= & 1 - (\alpha +(1 - \alpha) \lambda_2(M)) \nonumber \\ =& (1 - \alpha) \gamma(M).\nonumber
\end{align}
\end{remark}

\section{Appendix}

\subsection{\texorpdfstring{$\ell_1$}{1}-projection on $\Delta_d$}\label{app; 1}

In Example \ref{ex; 722}(b) it was necessary to project matrices on $\M_d$ with respect to $||\cdot||_\infty$. Since in this norm each row is considered separately, this problem is equivalent to the problem of projecting each of the matrices rows on $\Delta_d$ with respect to the $\ell_1$ norm. Thus, we need to solve (at most) $d$ optimization problems of the following form: For $x\in\R^d$ find an $\ell_1$-projection $P_{\Delta_d}(x)$ of $x$ on $\Delta_d$ (cf. \citet[p. 397]{boyd2004convex}): \begin{equation}\label{eq; 564}
P_{\Delta_d}(x)=\argmin_{y\in\Delta_d}\{||y-x||_1\}\end{equation} where  $$||z||_1 = \sum_{i=1}^d|z|,\; \forall z\in\R^d.$$ Notice that, in contrast to $||\cdot||_p$ for $p>1$, optimization problem (\ref{eq; 564}) has, in general, infinitely many solutions and we will understand $P_{\Delta_d}(x)$ as the set of all such solutions.

The following lemma seems to be well known but we were not able to find a reference. In it, only parts (a) and (b) and $|S|\leq 1$ are relevant for our needs.

\begin{lemma}\label{lem; 669}
Let $x=(x_1,\ldots,x_d)\in\R^d$. Denote $$S=\{i\in[d]\mid x_i<0\} \text { and } s = \sum_{i\in[d]\setminus S}x_i.$$ If $S=[d]$ then choose any $y\in\Delta_d$. Otherwise, define $y=(y_1,\ldots,y_d)\in\Delta_d$ as follows: Set $y_i = 0$ for every $i\in S$. Now, for every $i\in[d]\setminus S$:
\begin{enumerate}
    \item [(a)] If $s=1$ then set $y_i = x_i$ 
    \item [(b)] If $s<1$ then choose any $y_i \geq x_i$ such that $\sum_{j\in[d]\setminus S} x_j= 1$.
    \item [(c)] If $s>1$ then choose any $y_i \leq x_i$ such that $\sum_{j\in[d]\setminus S} x_j= 1$. 
 
\end{enumerate}
Then $y\in P_{\Delta_d}(x)$.
\end{lemma}
\begin{proof}
First assume $S=[d]$ and let $y'=(y'_1,\ldots,y'_d)\in\Delta_d$. For each $i\in [d]$ denote $\eps_i = y'_i - y_i$. Notice that $\sum_{i=1}^d\eps_i = 0$. It holds $$\sum_{i=1}^d|y'_i-x_i|= \sum_{i=1}^d(y'_i-x_i) = \sum_{i=1}^d(y_i+\eps_i-x_i)
=\sum_{i=1}^d(y_i-x_i) +\sum_{i=1}^d\eps_i =\sum_{i=1}^d|y_i-x_i|.$$
Now assume $S\neq[d]$. Let $z=(z_1,\ldots,z_d)\in\Delta_d$. We consider each of the three possibilities for $s$ separately:
\begin{enumerate}
    \item [(a)] Without loss, there exists $k\in[d]$ such that $x_i < 0$ for every $1\leq i\leq k$ and $x_i\geq 0$ for every $k+1\leq i\leq d$.  It holds 
    \begin{align}
    \sum_{i=1}^d|z_i-x_i|=&\sum_{i=1}^{k}|z_i-x_i|+\sum_{i=k+1}^{d}|z_i-x_i|\nonumber\\\geq&\sum_{i=1}^{k}|0-x_i|+\sum_{i=k+1}^{d}|x_i-x_i|   \nonumber\\=&\sum_{i=1}^{k}|y_i-x_i|+\sum_{i=k+1}^{d}|y_i-x_i|\nonumber\\ =& \sum_{i=1}^d|y_i-x_i|.  \nonumber
    \end{align}
    \item [(b)] Without loss, there exist $1\leq k\leq l\leq m\leq d$ such that 
    \begin{align}
        x_i<0=y_i\leq z_i, & \;\;\forall 1\leq i\leq k \nonumber\\
        0\leq z_i\leq x_i\leq y_i, & \;\;\forall k+1\leq i\leq l \nonumber\\
        0\leq x_i\leq z_i\leq y_i, & \;\;\forall l+1\leq i\leq m \nonumber\\
        0\leq x_i\leq y_i\leq z_i, & \;\;\forall m+1\leq i\leq d \nonumber.
    \end{align} Then, 
    \begin{align}
        \sum_{i=1}^{d}|z_i-x_i|  
        =& \sum_{i=1}^{d}|y_i-x_i|+\sum_{i=1}^{k}(z_i-y_i)  +\sum_{i=k+1}^{l}\left((y_i-z_i)-2(y_i-x_i)\right) -\nonumber\\ &\hspace{2cm}\sum_{i=l+1}^{m}(y_i-z_i)+\sum_{i=m+1}^{d}(z_i-y_i).
    \end{align}
    Thus, it suffices to show that \begin{equation}\label{eq; 334}
    \sum_{i=1}^{k}(z_i-y_i)+\sum_{i=k+1}^{l}\left((y_i-z_i)-2(y_i-x_i)\right)  -\sum_{i=l+1}^{m}(y_i-z_i)+\sum_{i=m+1}^{d}(z_i-y_i)\geq0.
    \end{equation} Indeed,
    \begin{align}
        (\ref{eq; 334}) \iff& \overbrace{\sum_{i=1}^{d}z_i}^{=1}+\sum_{i=k+1}^{l}\left(-2(-x_i+z_i)\right) 
        -\overbrace{\sum_{i=k+1}^{d}y_i}^{=1}\geq0 \nonumber\\\iff& \sum_{i=k+1}^{l}(x_i-z_i)\geq0\end{align}
        
        and, by assumption, $x_i\geq z_i$ for every $k+1\leq i\leq l$.
    \item [(c)] Similar to the previous case.
\end{enumerate}
\end{proof}

An ordinary triangle inequality trick would have introduced an additional $2$-factor on the sample complexity in Example \ref{ex; 722}(b). The following lemma shows that this may be avoided:

\begin{lemma}\label{lem; 567}
Suppose $(y_1,\ldots,y_d)\in\Delta_d$ and let $x=(x_1,\ldots,x_d)\in\R^d$ such that $||x||_1=1$. Let $\eps>0$ and assume that $||x-y||_1<\eps$. Then there exists $x'=(x'_1,\ldots,x'_d)\in P_{\Delta_d}(x)$ such that $||x'-y||_1<\eps$. 
\end{lemma}

\begin{proof}
If $x_1,\ldots,x_d\geq 0$ then $x\in\Delta_d$ and we may take $x'=x$. Otherwise, assume without loss that there is $1\leq k\leq d$ such that $x_i<0$ for $1\leq i\leq k$ and $x_i\geq 0$ for $k+1\leq i\leq d$. Let $k+1\leq l\leq d$ be minimal such that $\sum_{i=1}^l x_i\geq 0$ (such $l$ must exist since $||x||_1=1$). Define $x'=(x'_1,\ldots,x'_d)\in\Delta_d$ as follows: For $1\leq i\leq d$ let
$$ x'_i=
\begin{cases}
			0, & \text{if } 1\leq i\leq l-1 \\
             \sum_{j=1}^l x_j, & \text{if } i=l \\
            x_i, & \text{if } l+1\leq i\leq d.
		 \end{cases}
$$ It holds
\begin{align}
\sum_{i=1}^d|x'_i-y_i| = &\sum_{i=1}^{l-1}y_i + |\sum_{i=1}^l x_i- y_l| + \sum_{i=l+1}^d|x_i-y_i| \nonumber\\ \leq& \sum_{i=1}^{l-1}y_i -  \sum_{i=1}^{l-1}x_i + \sum_{i=l}^d|x_i-y_i|\nonumber\\\leq& \sum_{i=1}^d|x_i-y_i|<\eps\nonumber.
\end{align}
\end{proof}

\subsection{Proof of Lemma \ref{lem; 440}}\label{proof; 1}

\begin{proof}
Let $\pi$ be the stationary distribution of $M$ and suppose $(X_t)_{t\in\N}\sim(M,\mu)$. Let $(B_t)_{t\in\N}$ be a sequence of i.i.d. random variables, independent from $(X_t)_{t\in\N}$, such that $B_t \sim \text{Bernoulli}(1-\alpha)$ for each $t\in\N$. For $t\geq 0$ define $N_t = \sum_{i=1}^t B_i$ and $X'_t =X_{N_t}$. Clearly, $(X'_t)_{t\in\mathbb{N}}\sim(\lazy_\alpha(M),\mu)$.

For $i\in[d]$ and $t_0\in\N$ to be determined later we have for $t\geq t_0$:
\begin{align}
 ||(\lazy_\alpha(M))^t(i,\cdot) - \pi||\TV 
&= ||\sum_{n=0}^t \PP(N_t=n)\M^n(i,\cdot)-\pi ||\TV\nonumber\\ &\leq  \sum_{n=0}^{t_0} \PP(N_t=n)||M^n(i,\cdot)-\pi ||\TV+\nonumber\\ &\hspace{2cm}\sum_{n=t_0+1}^t \PP(N_t=n)||M^n(i,\cdot)-\pi ||\TV \nonumber \\ &\leq \PP(N_t\leq t_0) + \max_{n> t_0}\{||M^n(i,\cdot)-\pi ||\TV\}. \nonumber
\end{align} Consider the second term in the last expression: Since the total variation distance is monotone decreasing (cf. \citet[Proposition 7]{lalley2009convergence}), we have $$\max_{n> t_0}\{||M^n(i,\cdot)-\pi ||\TV\}= ||M^{t_0+1}(i,\cdot)-\pi ||\TV.$$ Thus, for $t_0\geq \tmix(M,\eps/2)$ it holds $$\max_{n> t_0}\{||M^n(x,\cdot)-\pi ||\TV\}\leq \eps/2.$$
Now, by Hoeffding's inequality, if $t_0 = (1-\alpha-\delta)t$ for some $\delta>0$ then $$\PP(N_t\leq (1-\alpha-\delta)t)\leq e^{-2\delta^2t}.$$ Taking $\delta=\alpha/2$ we obtain $$\PP(N_{2t_0/(1-\alpha)}\leq t_0)\leq e^{-2(1-\alpha) t_0}.$$ It follows that if $t_0 \geq \frac{\ln\frac{2}{\eps}}{1-\alpha}$ then $\PP(N_{2t_0/(1-\alpha)}\leq t_0)\leq\eps/2$. Thus, we may take $$t_0 = \max\left\{ \frac{\ln\frac{2}{\eps}}{1-\alpha},\tmix(M,\eps/2)\right\}.$$ The last assertion follows from \citet[(4.36) on p. 55]{levin2017markov}. 
\end{proof}

\bibliography{bibliography} 
\bibliographystyle{abbrvnat}
\end{document}